\documentclass[letterpaper, 10 pt, conference]{ieeeconf}  

\IEEEoverridecommandlockouts                              

\usepackage{cite}
\usepackage{amsmath,amssymb,amsfonts}
\usepackage{algorithmic}
\usepackage{graphicx}
\usepackage{textcomp}
\usepackage{xcolor}
\usepackage{blindtext}
\usepackage{graphicx}
\usepackage{caption}

\usepackage{lipsum}
\usepackage{mwe}
\usepackage{babel}
\usepackage[font=small,labelfont=bf,tableposition=top]{caption}

\usepackage{cuted}
\newtheorem{definition}{Definition}
\newtheorem{lemma}{Lemma}
\newtheorem{problem}{Problem}
\newtheorem{theorem}{Theorem}

\allowdisplaybreaks

\def\be{\boldsymbol e}

\def\bu{\boldsymbol u}
\def\bx{\boldsymbol x}

\def\bs{\boldsymbol s}
\def\bh{\boldsymbol h}

\title{\LARGE \bf
Non-cooperative Stochastic Target Encirclement by Anti-synchronization Control via Range-only Measurement 
}

\author{Fen Liu,~Shenghai Yuan,~Wei Meng,~Rong Su, ~Lihua Xie, ~\textit{Fellow,~IEEE}
\thanks{This work was partially supported by the National Natural Science Foundation of China (U21A20476, 62121004), Guangdong Introducing Innovative and Entrepreneurial Teams (2019ZT08X340) of Guangdong Province.}
\thanks{F. Liu, W. Meng are with Guangdong Provincial Key Laboratory of Intelligent Decision and Cooperative Control, School of Automation, Guangdong University of Technology, Guangzhou 510006.
        {\tt\small liufen7536@163.com, meng0025@ntu.edu.sg.}}%
\thanks{S. Yuan, R. Su, L. Xie are with the Centre for Advanced Robotics Technology Innovation (CARTIN), School of Electrical and Electronic Engineering, Nanyang Technological University, Singapore 639798, Singapore.
        {\tt\small shyuan@ntu.edu.sg, rsu@ntu.edu.sg, elhxie@ntu.edu.sg.}}%
}

\begin{document}


\maketitle
\begin{figure*}
    \centering
    \includegraphics[width=0.8\textwidth]{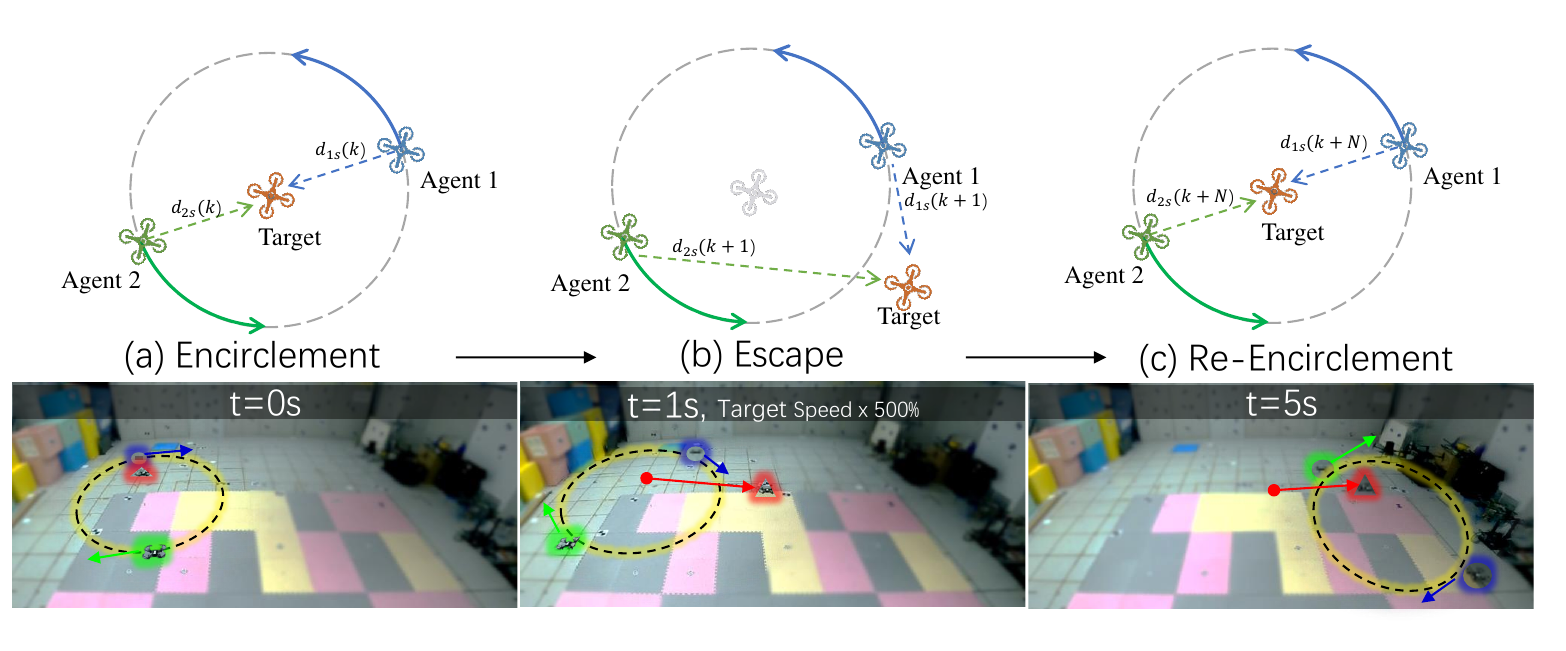}
    \caption{Overall illustration of proposed work.}
    \setlength{\belowcaptionskip}{-10pt}
    \label{as}
\end{figure*}

\begin{abstract}
This paper investigates the stochastic moving target encirclement problem in a realistic setting. In contrast to typical assumptions in related works, the target in our work is non-cooperative and capable of escaping the circle containment by boosting its speed to maximum for a short duration. Considering the extreme environment, such as GPS denial, weight limit, and lack of ground guidance, two agents can only rely on their onboard single-modality perception tools to measure the distances to the target. The distance measurement allows for creating a position estimator by providing a target position-dependent variable. Furthermore, the construction of the unique distributed anti-synchronization controller (DASC) can guarantee that the two agents track and encircle the target swiftly. The convergence of the estimator and controller is rigorously evaluated using the Lyapunov technique. A real-world UAV-based experiment is conducted to illustrate the performance of the proposed methodology in addition to a simulated Matlab numerical sample. Our video demonstration can be found in the URL \textcolor[rgb]{0.00,0.0,1.00}{https://youtu.be/JXu1gib99yQ}.
 \end{abstract}

\section{Introduction}
Target encirclement has a wide range of practical civil and military applications, such as convoy protection, target surveillance, and criminal entrapment \cite{Dong2020Target,jiang2019line,Yu2021Cooperative,peng2020event,zhang2022multi}. The tasking agents may have a beacon, ranging device or prior knowledge of the target but can not control the behavior of the target \cite{Zuo2019Time,Thien2020Single,peng2020moving,Jiang2017Simultaneous,aranda2014three}. Generally, the target state estimators need to be designed carefully so that the agent can make use of the state and plan the path accordingly \cite{Hua2019Distributed,Cao2011Formation,Li2018Localization}. The state of the target can be estimated by various methods, such as angle-based feedback, distance-based algorithms, or even an indoor motion-tracking system \cite{Shames2012Circumnavigation,Yu2018Distributed,Thien2020Persistently,guo2019ultra}. Based on a reasonable state estimator, the tasking agent can encircle the target using various decision and control models.

In most encirclement works, many assumptions have to be made on the non-cooperative target, such as stationary, constant speed, and low speed \cite{Dong2017Time,Dong2020Flight,guo2010cooperative}. Very few works touch on stochastic motion types \cite{zhang2022multi,dou2021moving,dong2017necessary,nguyen2018robust}. In some real-world situation, the target may be able to shortly boost the escaping capabilities to the maximum by using equipment like nitrous oxide gas to accelerate the car's (target) speed to escape the police (agents) chase. In such cases, the encirclement problem can often become a pursuit and evade problem, and it is difficult for the agents to continue encircling the target.

In this paper, the encirclement problem of non-cooperative targets is investigated, which is shown as Fig. \ref{as}. Our work is applicable to a large variety of mobile robotics systems, such as fixed-wing aircraft, multi-rotor UAVs, ground robotics systems, and unmanned surface vehicles. Our solution can be integrated with any UWB or cameras on normal mobile robotics platforms and can also work with high-latency miniature-sized prototypes such as Tello drones. We validated our results with MATLAB simulations for UGV agents and real-world demonstrations with a swarm of Tello drones without the need for any motion-capturing system or GPS.

The main contributions of this work are summarized as:
\begin{enumerate}
  \item To the best of our knowledge, we are the first to study the encirclement of dynamic stochastic moving targets. These targets simulate the real-world combined problem of encirclement, pursuit, and evasion by assuming that the target can increase its speed to the maximum in order to evade the encirclement.
  \item The position of the non-cooperative target can be approximately estimated by using only two distance measurements made by the two agents using their own sensors. No other real-time information about the target is required.
  \item Based on the principle of anti-synchronization (AS) \cite{Liu2021Anti,Liu2022Bounded}, a new distributed anti-synchronization controller (DASC) is designed so that the two agents can quickly localize, track, and encircle the target symmetrically.
\end{enumerate}

\section{Problem formulation}
In this work, two agents and one non-cooperative target are considered in a GPS-denied environment. In practical application, it can be extended to multiple pairs of agents encircling multiple targets. Without loss of generality, the two agents are defined as Agent 1 and Agent 2. The stochasticity of the target includes two meanings. The first refers to the stochasticity of the target speed, that is, the target can move at any low speed for most of the time. The second is that the target can move at a high speed at stochastic moments. Furthermore, suppose that two agents can obtain their own displacement and measure the distance between each other and the distance to the target through their own sensors.

Considering several common types of agents, e.g., unmanned ground vehicle (UGV), the approximated model can be obtained as following formulas.
\begin{equation}\label{eq1-1}
\left\{\begin{split}
x_i(k+1)=&x_i(k)+\triangle x'_i(k)\cos(\phi_i(k))\\
&-\triangle y'_i(k)\sin(\phi_i(k)),\\
y_i(k+1)=&y_i(k)+\triangle x'_i(k)\sin(\phi_i(k))\\
&+\triangle y'_i(k)\cos(\phi_i(k)),\\
\end{split}\right.
\end{equation}
where $x_i(k)$, $y_i(k)$ represent the positions of the agent $i,  i\in \Phi_1\triangleq\{1,2\}$ in the $X$, $Y$ axes of the global coordinate system at $k$ time, respectively. $\phi_i(k)$ is the yaw of the agent $i$. $\Delta x'_i(k)$ and $\Delta y'_i(k)$ are the change values of position in the local coordinate system.

Based on the above model and the assumption that the change values in the position and angle of the agents are determined by external inputs, the dynamical model of two agents can be defined as
\begin{equation}\label{eq1-2}
\begin{split}
&\bx_{i}(k+1)=\bx_{i}(k)+ \bu_{i}(k),
\end{split}
\end{equation}
where $\bx_i(k)=[x_i(k),y_i(k)]^T$,
and $\bu_i(k)=[\bu_{ix}(k),\bu_{iy}(k)]^T$ are the position and the controlled input of the agent, respectively.

Consider that the non-cooperative target generally moves at a low speed, and high-speed displacements will occur at some random moments. Regardless of the motion angle of the target, the dynamic of the non-cooperative target is described as
\begin{equation}\label{eq1-3}
\left\{\begin{split}
&\bs(k+1)=\bs(k)+\bh(k),\\
&\bh(k)=\nu(k)+\sum_{m=1}^\infty\{\theta(k)\}\delta(k-k_m),
\end{split}\right.
\end{equation}
where $\bs(k)=[\bs_x(k),\bs_y(k)]$ is the state of target at the time instant $k$. The term $\bh(k)$ is the unknown displacement of target, in which $\nu(k)<\bar{\nu}$ indicates the low speed part and $\sum_{m=1}^\infty\{\theta(k)\}\delta(k-k_m), m \in Z^+ $ represents the stochastic high-speed displacement. $\delta(k-k_m)$ is a stochastic impulsive function, in which $\delta(k-k_m)=1$ for $k=k_m$, and $\delta(k-k_m)=0$ for $k\neq k_m$. The initial impulsive $k_1=0$. The interval between two continuous impulses $k_{m}$ and $k_{m-1}$ is defined as $\ell_m$, and $\ell_m>\ell$. $\theta(k)<\bar{\theta}$ is an unknown random variable, and can take on very large values so as to break through the encirclement of the agents.

Then, a problem with non-cooperative target encirclement is stated as follows:

\begin{problem}
In this work, the target is non-cooperative, completely unknown, and at stochastic speed, that is, its speed, direction of movement, and location are not directly known to the agent and  has the capability to boost its speed to the maximum to evade encirclement. Therefore, the two agents can  achieve the following goals only through the distance information acquired by their own onboard sensors.
\begin{enumerate}
  \item  The two agents can obtain the approximate position of the target by designing an estimator.
  \item The two agents can persistently and symmetrically navigate around a circle centered on the target by the designed controller.
\end{enumerate}
\end{problem}

\section{Estimator and controller design}
Denote the self-displacement $\psi_i(k)$ of the agent, the distance $d_{12}(k)$ between two agents and the distance $d_{is}(k)$ from the $i$-$th$ agent to the target as
\begin{equation}\label{eq1-4}
\left\{\begin{split}
&\psi_i(k)=\bx_i(k+1)-\bx_i(k),\\
&d_{12}(k)=\|\bx_1(k)-\bx_2(k)\|,\\
&d_{is}(k)=\|\bx_i(k)-\bs(k)\|,i\in \Phi_1.\\
\end{split}\right.
\end{equation}

In order to estimate the state of the target, the following variable that is related to the position of the target can be obtained by calculating the formula $d_{1s}^2(k)-d_{2s}^2(k)$,
\begin{equation}\label{eq1-5}
\begin{split}
\varpi(k)=&p_{12}^T(k)\bs(k)\\
=&-\frac{1}{2}(d_{1s}^2(k)-d_{2s}^2(k)-\bx_1^T(k)\bx_1(k)\\
&+\bx_2^T(k)\bx_2(k)),
\end{split}
\end{equation}
where $p_{12}(k)=\bx_1(k)-\bx_2(k)$ is the relative position from Agent $1$ to Agent $2$.

\subsection{Estimator design}
Based on the variable $\varpi(k)$, the dynamic of the target position estimator (TPE) can be designed as
\begin{equation}\label{eq1-6}
\begin{split}
\hat{\bs}(k+1)=\hat{\bs}(k)+K(k+1)(\varpi(k+1)-p_{12}^T(k+1)\hat{\bs}(k)),
\end{split}
\end{equation}
where $\hat{\bs}(k)$ is the estimated state of target position at instant $k$. $K(k+1)$ is the estimator gain and can be further described as follows by the least squares fit,
\begin{equation}\label{eq1-7}
\begin{split}
K(k)&=\frac{\eta(k-1)p_{12}(k)}{\gamma_1\gamma_2+p_{12}^T(k)\eta(k)p_{12}(k)},
\end{split}
\end{equation}
where the covariance matrix $\eta(k) \in \mathrm{R}^{n\times n}$ $(\eta(k)=\eta^T(k)>0)$ is defined as
\begin{equation}\label{eq1-8}
\begin{split}
\eta^{-1}(k)=&\gamma_1\eta^{-1}(k-1)+\frac{1}{\gamma_2}p_{12}(k)p_{12}^T(k),
\end{split}
\end{equation}
$\gamma_1 \in [0,1]$ is an exponential forgetting factor and $\gamma_2 \in [0,1]$ is a new information utilization factor.

The formulation $\|\zeta(r,k)\|=r$ can be obtained. Since $p_{12}(k+1)=p_{12}(k)+\bu_{1}(k)-\bu_{2}(k)$, we have $\|\bu_{1}(k)-\bu_{2}(k)\|\leq\bar{\mu}$ with $\bar{\mu}=|\alpha|(d_{12}(0)+2r)$ for $0<|1+\alpha|<1$. 

The estimation error is defined as $\hat{\be}(k)$. Then, according to the above estimator \eqref{eq1-6}, the dynamic of $\hat{\be}(k)$ can be further obtained as
\begin{equation}\label{eq1-9}
\hat{\be}(k+1)=
\left\{\begin{split}
&A(k)(\hat{\be}(k)+\nu(k)),k\neq k_m,\\
&A(k)(\hat{\be}(k)+\nu(k)+\theta(k)),k=k_m,
\end{split}\right.
\end{equation}
where $A(k)=I-K(k+1)p_{12}^T(k+1)$.

\subsection{Controller design}
Based on the TPE designed above, the DASC can be designed as
\begin{equation}\label{eq1-10}
\left\{\begin{split}
\bu_{1}(k)=&\alpha(\hat{p}_{10}(k)+\zeta(r,k)),\\
\bu_{2}(k)=&\alpha(\hat{p}_{20}(k)-\zeta(r,k)),\\
\end{split}\right.
\end{equation}
where $\alpha$ is controller gain that needs to be designed and $\hat{p}_{10}(k)=\bx_1(k)-\hat{\bs}(k)$. $\zeta(r,k)\in \mathrm{R}^n$ is the preset trajectory that makes Agent 1 and Agent 2 circumnavigate around the target, which satisfies $\zeta(r,k)=\zeta(r,k+\frac{2}{\nu})$, $r$ denotes the minimum radius of the Agent 1 and Agent 2's trajectories, $0<\nu<1$ denotes the frequency of circumnavigation. 

Recalling model \eqref{eq1-1}, we can have
\begin{equation}\label{eq1-11}
\left\{\begin{split}
\triangle x'_i(k)=&-\bu_{ix}\sin(\phi_i(k))+\bu_{iy}\cos(\phi_i(k)),\\
\triangle y'_i(k)=&\bu_{ix}\cos(\phi_i(k))+\bu_{iy}\sin(\phi_i(k)),\\
\phi_{i}(k)=&arc\tan2\{\hat{\bs}_y(k)-y_i(k),\hat{\bs}_x(k)-x_i(k)\},
\end{split}\right.
\end{equation}
where $arc \tan2 \{x, y\}$ calculates a unique arc tangent value from two variables $x$ and $y$.

Based on the models \eqref{eq1-2}, \eqref{eq1-3} and the DASC \eqref{eq1-10}, the AS error is defined as $\be_s(k)=\bx_1(k)-\bs(k)+\bx_2(k)-\bs(k)$. Furthermore, the dynamic of $\be_s(k)$ can be deduced as
\begin{equation}\label{eq1-12}
\be_s(k+1)=\left\{\begin{split}
&(1+\alpha)\be_s(k)+2\alpha\hat{\be}(k)-2\nu(k),\\
&~~~~k\neq k_m,\\
&(1+\alpha)\be_s(k)+2\alpha\hat{\be}(k)-2(\nu(k)\\
&~~~~+\theta(k)),k=k_m.
\end{split}\right.
\end{equation}

\subsection{Convergence analysis}
At first, some lemmas and definitions are given which are useful for analyzing error convergence.

\begin{definition}
Agent 1 and Agent 2 can achieve target centered AS if the following conditions hold,
\begin{equation}\label{eq1-13}
\begin{split}
&\underset{k\rightarrow \infty} {\mathrm{lim}}\{p_{1s}(k)+p_{2s}(k)\}=0,
\end{split}
\end{equation}
where $p_{is}(k)=\bx_i(k)-\bs(k), i\in \Phi_1$.
\end{definition}

\begin{lemma}\cite{Johnstone1982Exponential}
The sequence $\{p_{12}(k)\}, k\in[M,M+N-1], \forall M\in Z$ is persistently exciting for $ \|\bu_1(k)-\bu_2(k)\|\leq\overline{\mu}$, that is, $0<\hat{\vartheta}I_n\leq\sum_{k=M}^{M+N-1}p_{12}(k) p_{12}^T(k)\leq\check{\vartheta}I_n<\infty$, then the following condition holds,
\begin{equation}\label{eq1-14}
\left\{\begin{split}
&\eta^{-1}(k)\geq\hat{b},~\forall k\geq N-1,\\
&\eta^{-1}(k)\leq\check{b},~\forall k\in Z,\\
\end{split}\right.
\end{equation}
where
\begin{equation*}
\begin{split}
\hat{b}=&\frac{N(1-\frac{1}{\gamma_1})}{\gamma_2(1-\frac{1}{\gamma_1^N})}\hat{\vartheta}I_n,\\
\check{b}=&\frac{1-\gamma_1}{1-\gamma_1^N}\sum_{k=1}^{N-1}\eta^{-1}(k)+\frac{N}{(1-\gamma_1^N)\gamma_2}\check{\vartheta}I_n,
\end{split}
\end{equation*}
and $N$ is the motion period that occurs when the two tasking agents are moving around the target. $ {I_n}$ is the ${n}$-dimensional identity matrix. 
\end{lemma}

\begin{theorem}
Under Lemma 1, the estimator \eqref{eq1-6} can approximately estimate the target position if the following condition holds by selecting appropriate factors $\gamma_1$,
\begin{equation}\label{eq1-15}
\begin{split}
0<2\gamma_1\leq\frac{2}{3}.
\end{split}
\end{equation}
\end{theorem}

\begin{proof}
The following LF (Lyapunov function) can be chosen as
\begin{equation}\label{eq1-16}
\begin{split}
V_1(k)=\hat{\be}^T(k)\eta^{-1}(k)\hat{\be}(k).
\end{split}
\end{equation}


According to the matrix inversion lemma, $A(k)=\gamma_1\eta(k+1)\eta^{-1}(k)$. Furthermore, based on the estimator error $\hat{\be}(k+1)$ in \eqref{eq1-9} and the Cauchy-Schwarz inequality, the difference of the LF can be deducted as
\begin{equation}\label{eq1-18}
\triangle V_1(k)\leq\left\{\begin{split}
&2\gamma_1\hat{\be}^T(k)\eta^{-1}(k)\hat{\be}^T(k)+2\gamma_1\nu^T(k)\eta^{-1}(k)\\
&\times\nu(k)-\hat{\be}^T(k)\eta^{-1}(k)\hat{\be}(k),~~k\neq k_m,\\
&3\gamma_1\hat{\be}^T(k)\eta^{-1}(k)\hat{\be}^T(k)+3\gamma_1\nu^T(k)\eta^{-1}(k)\\
&\times\nu(k)+3\gamma_1\theta^T(k)\eta^{-1}(k)\theta(k)\\
&-\hat{\be}^T(k)\eta^{-1}(k)\hat{\be}(k),~~~~k=k_m.\\
\end{split}\right.
\end{equation}

Considering Lemma 1, when $k\neq k_m$, we have
\begin{equation}\label{eq1-19}
\begin{split}
V_1(k+1)\leq&2\gamma_1V_1(k)+\varrho_1,\\
\end{split}
\end{equation}
where $\varrho_1=2\gamma_1\bar{\nu}^2\check{b}$.

Next, when $k=k_m$, we have
\begin{equation}\label{eq1-20}
\begin{split}
V_1(k+1)\leq&3\gamma_1V_1(k)+\varrho_2,\\
\end{split}
\end{equation}
where $\varrho_2=3\gamma_1(\bar{\nu}^2+\bar{\theta}^2)\check{b}$.

Then, when $k=k_m+a, a\in\{0,1,\ldots,\ell-1\}$, we can obtain that
\begin{equation}\label{eq1-21}
\begin{split}
V_1(k+1)\leq&(2\gamma_1)^{(m-1)(\ell-1)+a}(3\gamma_1)^mV_1(0)+\varrho,\\
\end{split}
\end{equation}
which implies
\begin{equation}\label{eq1-22}
\begin{split}
||\hat{\be}(k+1)||^2\leq&\frac{(2\gamma_1)^{(m-1)(\ell-1)+a}(3\gamma_1)^m||\check{b}_{0}||}{||\hat{b}||}\\
&\times||\hat{\be}(0)||^2+\frac{\varrho}{||\hat{b}||},\\
\end{split}
\end{equation}
where $\frac{\varrho}{||\hat{b}||}$ is the error bound. The larger the value of $\ell$, the smaller the bound $\varrho$.

Furthermore, considering the condition \eqref{eq1-15}, $||\hat{\be}_0(k+1)||^2\leq \frac{\varrho}{||\hat{b}||}$ for $k\rightarrow\infty$, which indicates that the target position can be estimated approximately.
\end{proof}

\begin{theorem}
Under the TPE and the DASC, Agent 1 and Agent 2 can achieve symmetrically AS centered at the target when the controller gain $\alpha$ satisfies the following condition,
\begin{equation}\label{eq1-23}
\begin{split}
0<3(1+\alpha)^2\ell\leq\frac{3}{4}.
\end{split}
\end{equation}
\end{theorem}

\begin{proof}
As to certificate the convergence of the AS error, the LF can be selected as
\begin{equation}\label{eq1-324}
\begin{split}
V_2(k)=\be_s^T(k+1)\be_s(k+1).
\end{split}
\end{equation}

Next, based on the quality \eqref{eq1-12} and the condition $\lambda\{p_{ij}(k)p_{ij}^T(k)\}\leq\overline{\nu}$, the difference of the LF can be deduced as
\begin{equation}\label{eq1-25}
\Delta V_2(k)
\leq\left\{\begin{split}
&(3(1+\alpha)^2-1)\be_s^T(k)\be_s(k)+12\alpha^2\hat{\be}^T(k)\hat{\be}(k)\\
&+12\nu^T(k)\nu(k),~~k\neq k_m,\\
&(4(1+\alpha)^2-1)\be_s^T(k)\be_s(k)+16\alpha^2\hat{\be}^T(k)\hat{\be}(k)\\
&+16\nu^T(k)\nu(k)+16\theta(k)^T\theta(k)\\
&~~k= k_m.\\
\end{split}\right.
\end{equation}

Recalling the conclusions mentioned in Theorem 1, $\hat{\be}(k)\rightarrow0$ for $k\rightarrow\infty$. Then, when $m\rightarrow\infty$, we have
\begin{equation}\label{eq1-26}
V_3(k+1)
\leq\left\{\begin{split}
& 3(1+\alpha)^2V_3(k)+\sigma_1,k\neq k_m,\\
& 4(1+\alpha)^2V_3(k)+\sigma_2,k= k_m,\\
\end{split}\right.
\end{equation}
where $\sigma_1=12\bar{\nu}^2$ and $\sigma_2=16(\bar{\nu}^2+\bar{\theta}^2)$.

Furthermore, when $k=k_m+a,~a\in\{0,1,\ldots,\ell-1\}$, we have
\begin{equation}\label{eq1-27}
\begin{split}
V_2(k+1)\leq&(3(1+\alpha)^2)^{(m-1)(\ell-1)+a}(4(1+\alpha)^2)^mV_2(0)\\
&+\sigma,\\
\end{split}
\end{equation}
where $\sigma$ is the error boundary.

Considering the condition (22), the following result is established,
\begin{equation}\label{eq1-28}
\begin{split}
\lim_{k\rightarrow\infty}||\be_s(k+1)||^2\leq\sigma,
\end{split}
\end{equation}
which implies that the Agent 1 and Agent 2 can
circumnavigate all targets while maintaining AS within the error boundary $\sigma$.
\end{proof}
\section{SIMULATIONS AND EXPERIMENTS}
\textbf{Numerical simulation:}
To further prove that the designed estimator and DAS controller can ensure that Agent 1 and Agent 2 can symmetrically encircle a non-cooperative target that has the ability to break through the encirclement circle, a simulation example is presented.

Suppose that the displacement mode of the target is
\begin{equation*}
\begin{split}
\bh(k)=&\left[
         \begin{array}{c}
           0.02\cos(0.01k) \\
           0.02\sin(0.01k) \\
         \end{array}
       \right]\\
&+\sum_{m=1}^\infty\Big\{\left[
         \begin{array}{c}
           1.5rand(1) \\
           1.5rand(1) \\
         \end{array}
       \right]\Big\}\delta(k-k_m),
\end{split}
\end{equation*}
where the function $rand(1)$ is used to generate a random number evenly distributed between (0, 1). This random number is meant to simulate the short-term high-speed burst capability of the target to escape the encirclement. The minimum impulsive interval $\ell$ is set as $20$.

\begin{figure}
\centering
  \includegraphics[width=8cm]{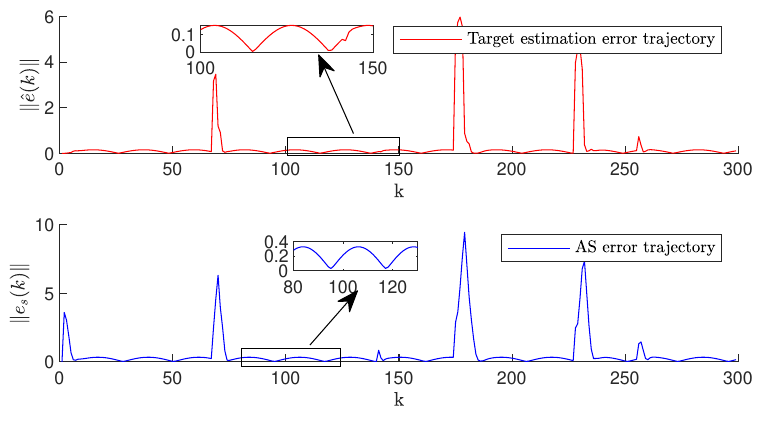}
 \caption{The trajectories of the target estimation error and the AS error.}
  \label{error}
\end{figure}

\begin{figure}
\centering
  \includegraphics[width=8cm]{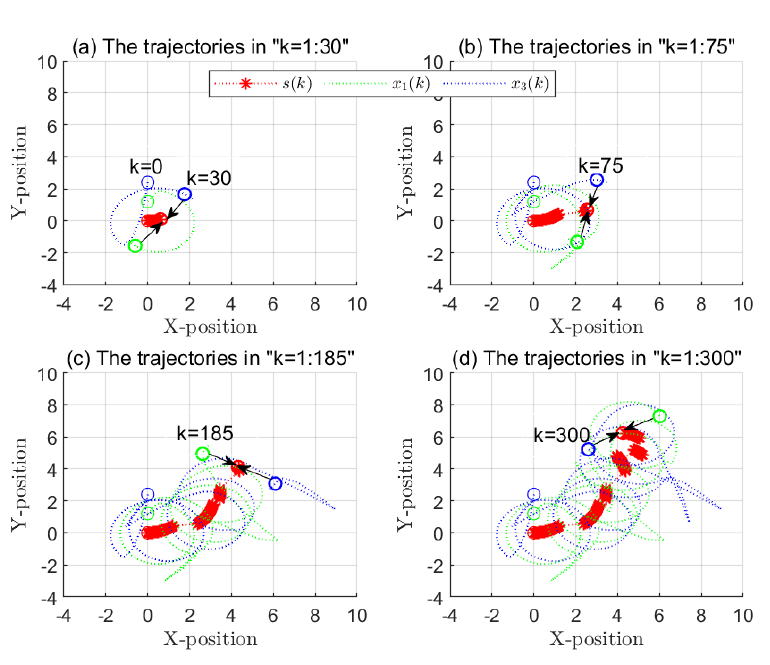}
 \caption{The trajectories of the two agents and the target for k=1:300.}
  \label{trajectory}
\end{figure}

Based on the Theorem 1 and Theorem 2, the exponential forgetting factor $\gamma_1$ is given as $0.3$,  the new information utilization factor $\gamma_2$ as $0.9$, the controller gains $\alpha$ as $-0.85$.

Furthermore, the upper and lower speed of the two agents can not exceed $-0.5$ and $0.5$, respectively. In this simulation, consider that all agents can convert their position in the local coordinate system to the global coordinate system. All initial positions of the two agents and the target in the global coordinate system are given as
\begin{equation*}
\begin{split}
&\bx_{1}(0)=[0,1.2]^T,~\bx_{2}(0)=[0,2.4]^T,\\
&\bs(0)=\hat{\bs}(0)=[0,0]^T,
\end{split}
\end{equation*}
and the preset trajectory as
\begin{equation*}
\begin{split}
&\zeta(k)=[2sin(\frac{k\pi}{24}),2cos(\frac{k\pi}{24})]^T.
\end{split}
\end{equation*}

Under the action of the target position estimator and the DAS controller, the simulation results can be obtained in Fig. \ref{error} - Fig. \ref{trajectory}. From Fig. \ref{error}, it is easy to see that the target estimation error and the AS error have suddenly increased when the target escapes, and then overall the errors have converged to a very small value, i.e, $\|\hat{\be}(k)\|\leq 0.12$ and $\|\be_s(k)\|\leq 0.38$. Fig. \ref{trajectory} has shown the real-time positions and the moving trajectories of all agents and target at the instants $k=30, 75, 185, 300$. When the target escapes from the circle of two agents, the two agents can quickly locate the target and form a circle around it again. In addition, the two agents can maintain AS to achieve sensor coverage of the target.

\textbf{Real-world UAV-based experiment:}
In the real-world UAV-based experiment, we assume the target is actively broadcasting its image feed so that the relative distance between agents and the target can be approximated. In order to fly multiple drones at a reasonable cost, commercial off-the-shelf (COTS) low-cost Tello drones are used to represent both the target and the following agents (UAV0 is the target, UAV1 and UAV2 are the two agents). The Tello drone was chosen because it is crash-resistant, has an open API for streaming camera feeds, and can be controlled in real-time. Because it is an 80-gram drone, no COTS distance measurement device such as an RGBD camera, Zigbee, or UWB can be directly installed. In addition to the take-off weight limit, an additional port forwarding mechanism is required because all of the drones feed to the same IP and ports, resulting in conflicts. Extra Raspberry PI 4s are deployed to receive the IP at the fixed IP/port and relay the data to another floating IP/port to deconflict the IP. 

\begin{figure}
\centering
  \includegraphics[width=8.5cm]{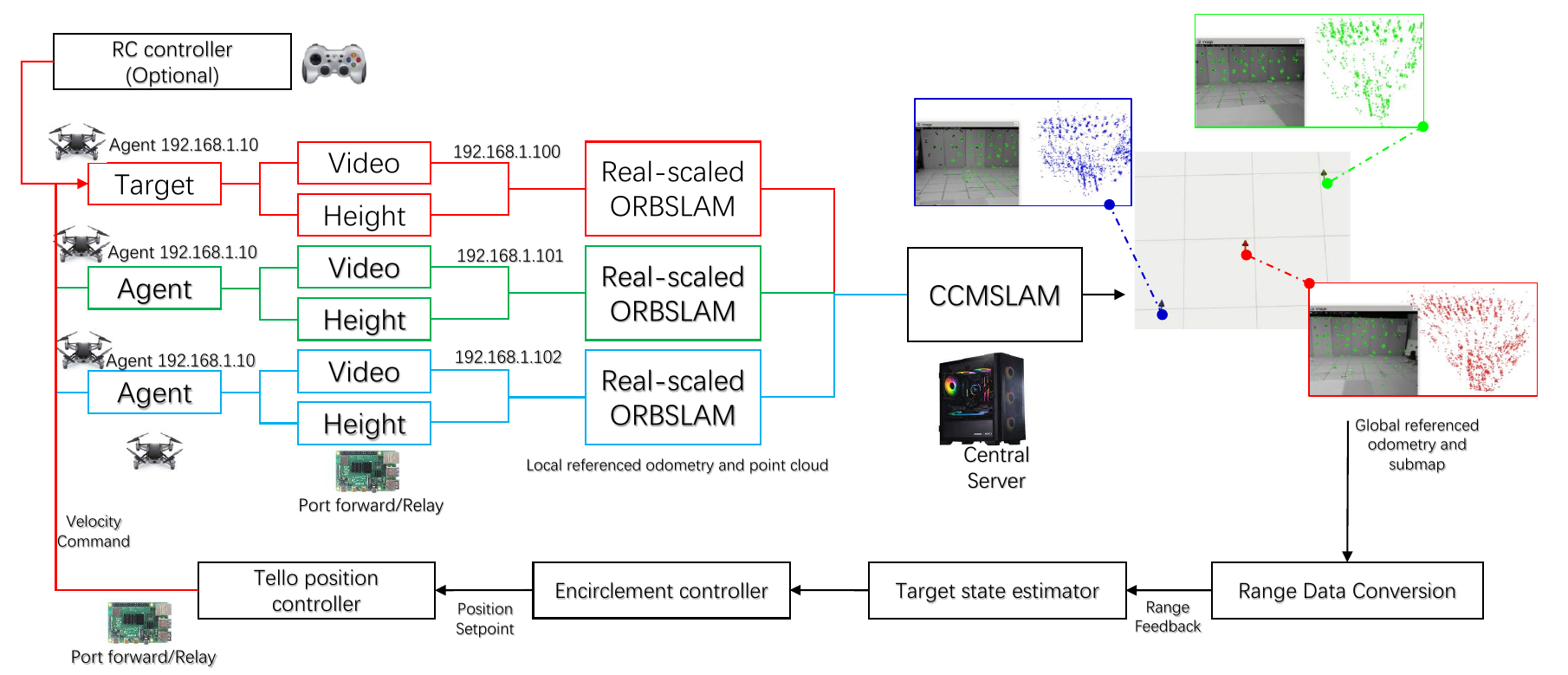}
 \caption{The overall experiment setup.}
  \label{fig:overallexperiment}
\end{figure}

The CCMSLAM\cite{schmuck2017multi} is used to find the relative pose relationship between each ORBSLAM \cite{murORB2} output point cloud in order to obtain range information by using 3D space expanded model \eqref{eq1-1}. Tello altitude measurement is used to adjust the point cloud scale. An overall global point cloud map can be created by minimizing the reprojection error of ORBSLAM submaps from each drone. The relative position between the individual submap starting point and current pose is known because each drone is running independent scale-aware ORBSLAM. As a result, the distance between each drone can be calculated. After obtaining the relative distances, the system performs the procedure described in the section to form the close-loop control. The overall experiment setup can be seen from Fig. \ref{fig:overallexperiment}.

\begin{figure}
\centering
  \includegraphics[width=7cm]{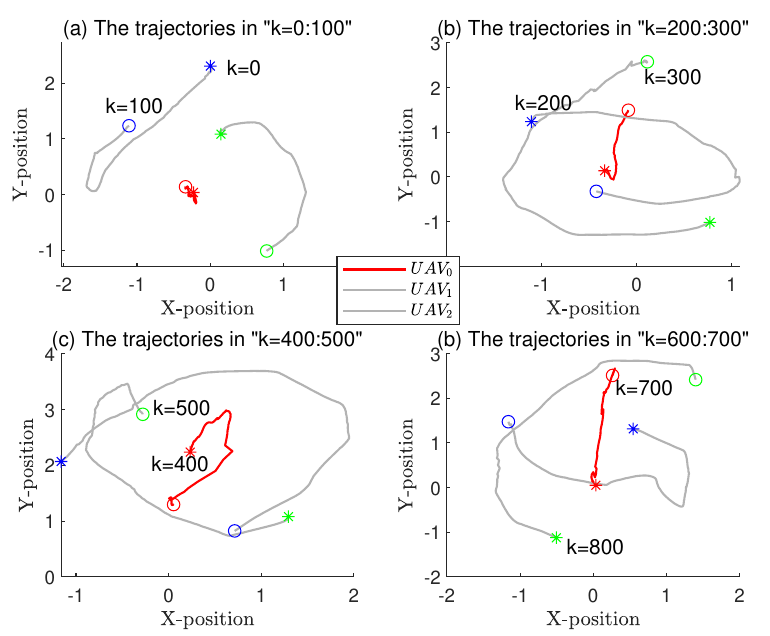}
 \caption{The trajectories of all UAVs in the different time periods when UAV0 (target) is given motion rule.}
 \captionsetup{belowskip=0pt}
  \label{auto}
\end{figure}

\begin{figure}
\centering
  \includegraphics[width=7cm]{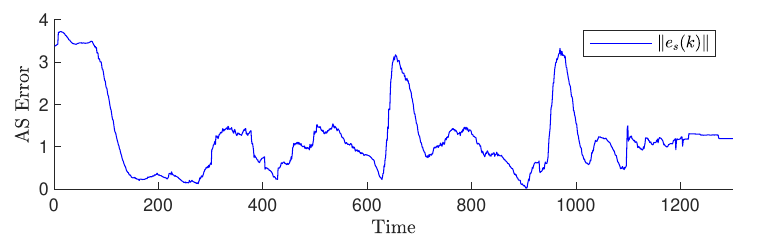}
 \caption{The AS error trajectories when UAV0 (target) is given motion rule.}
  \label{auto_error}
\end{figure}

 The target was controlled using two different methods to simulate escape: by manually controlling and by providing the same motion rules as the numerical simulation example. Part of the experiment results are shown in Fig. \ref{auto} - Fig. \ref{hand-error}, in which $\ast$ and $\circ$ denote the positions of all UAVs at the beginning and end of the time period, respectively. It is apparent that, while TPE and DASC are active, UAV1 and UAV2 can find, follow, and capture UAV0 relatively easily. The experiment's specifics are available online \textcolor[rgb]{0.00,0.0,1.00}{https://youtu.be/JXu1gib99yQ}.

\begin{figure}
\centering
  \includegraphics[width=7cm]{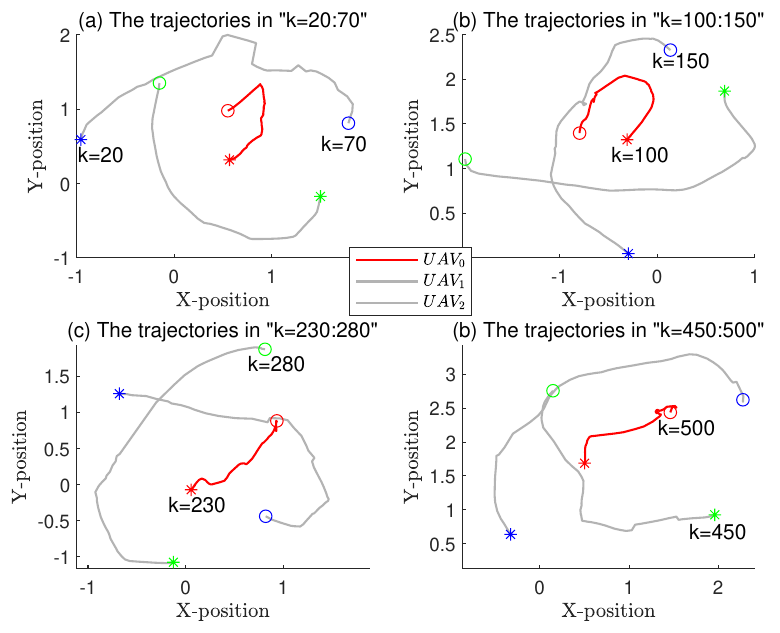}
 \caption{The trajectories of all UAVs in the different time periods when UAV0 (target) is under manually control.}
  \label{hand}
\end{figure}

\begin{figure}
\centering
  \includegraphics[width=7cm]{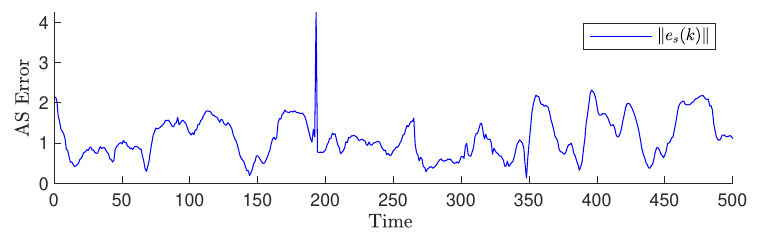}
 \caption{The AS error trajectories when UAV0 (target) is under manual control.}
  \label{hand-error}
\end{figure}

\section{Conclusions}
This study looked into the encirclement of non-cooperative targets by two agents. First, an approximation of the target position has been achieved to ensure that both agents can track the target.
Second, a controller algorithm has been created, which includes the algorithms for the two agents to surround and obstruct the target.Through careful analysis, Theorems 1 and 2 have shown that the proposed estimator and controller converge.
The efficiency of the controller has demonstrated by a numerical simulation and a real-world UAV-based experiment, which shows that the two agents can symmetrically encircle the target. In addition, a more realistic stochastic moving model of the target, e.g, the Brownian motion, and the collision avoidance of the agent can be further researched.





\bibliographystyle{IEEEtran}
\bibliography{IEEEfull}

\begin{thebibliography}{10}
\providecommand{\url}[1]{#1}
\csname url@samestyle\endcsname
\providecommand{\newblock}{\relax}
\providecommand{\bibinfo}[2]{#2}
\providecommand{\BIBentrySTDinterwordspacing}{\spaceskip=0pt\relax}
\providecommand{\BIBentryALTinterwordstretchfactor}{4}
\providecommand{\BIBentryALTinterwordspacing}{\spaceskip=\fontdimen2\font plus
\BIBentryALTinterwordstretchfactor\fontdimen3\font minus
  \fontdimen4\font\relax}
\providecommand{\BIBforeignlanguage}[2]{{%
\expandafter\ifx\csname l@#1\endcsname\relax
\typeout{** WARNING: IEEEtran.bst: No hyphenation pattern has been}%
\typeout{** loaded for the language `#1'. Using the pattern for}%
\typeout{** the default language instead.}%
\else
\language=\csname l@#1\endcsname
\fi
#2}}
\providecommand{\BIBdecl}{\relax}
\BIBdecl

\bibitem{Dong2020Target}
F.~Dong, K.~You, and S.~Song, ``Target encirclement with any smooth pattern
  using range-based measurements,'' \emph{Automatica}, 2020, DOI:
  10.1016/j.automatica.2020.108932.

\bibitem{jiang2019line}
Y.~Jiang, Z.~Peng, D.~Wang, and C.~P. Chen, ``Line-of-sight target enclosing of
  an underactuated autonomous surface vehicle with experiment results,''
  \emph{IEEE Transactions on Industrial Informatics}, vol.~16, no.~2, pp.
  832--841, 2019.

\bibitem{Yu2021Cooperative}
X.~Yu, J.~Ma, N.~Ding, and A.~Zhang, ``Cooperative target enclosing control of
  multiple mobile robots subject to input disturbances,'' \emph{IEEE
  Transactions on Systems, Man, and Cybernetics: Systems}, vol.~51, no.~6, pp.
  3440--3449, 2021.

\bibitem{peng2020event}
Z.~Peng, Y.~Jiang, and J.~Wang, ``Event-triggered dynamic surface control of an
  underactuated autonomous surface vehicle for target enclosing,'' \emph{IEEE
  Transactions on Industrial Electronics}, vol.~68, no.~4, pp. 3402--3412,
  2020.

\bibitem{zhang2022multi}
T.~Zhang, Z.~Liu, Z.~Pu, and J.~Yi, ``Multi-target encirclement with collision
  avoidance via deep reinforcement learning using relational graphs,'' in
  \emph{2022 IEEE International Conference on Robotics and Automation
  (ICRA)}.\hskip 1em plus 0.5em minus 0.4em\relax IEEE, 2022, pp. 8794--8800.

\bibitem{Zuo2019Time}
S.~Zuo, Y.~Song, F.~L. Lewis, and A.~Davoudi, ``Time-varying output
  formation-containment of general linear homogeneous and heterogeneous
  multi-agent systems,'' \emph{IEEE Transactions on Control of Network
  Systems}, vol.~6, no.~2, pp. 537--548, 2019.

\bibitem{Thien2020Single}
T.~M. Nguyen, Z.~Qiu, M.~Cao, T.~H. Nguyen, and L.~Xie, ``Single landmark
  distance-based navigation,'' \emph{IEEE Transactions on Control Systems
  Technology}, vol.~28, no.~5, pp. 2021--2028, 2020.

\bibitem{peng2020moving}
X.~Peng, K.~Guo, and Z.~Geng, ``Moving target circular formation control of
  multiple non-holonomic vehicles without global position measurements,''
  \emph{IEEE Transactions on Circuits and Systems II: Express Briefs}, vol.~67,
  no.~2, pp. 310--314, 2020.

\bibitem{Jiang2017Simultaneous}
B.~Jiang, M.~Deghat, and B.~D.~O. Anderson, ``Simultaneous velocity and
  position estimation via distance-only measurements with application to
  multi-agent system control,'' \emph{IEEE Transactions on Automatic Control},
  vol.~62, no.~2, pp. 869--875, 2017.

\bibitem{aranda2014three}
M.~Aranda, G.~L{\'o}pez-Nicol{\'a}s, C.~Sag{\"u}{\'e}s, and M.~M. Zavlanos,
  ``Three-dimensional multirobot formation control for target enclosing,'' in
  \emph{2014 IEEE/RSJ International Conference on Intelligent Robots and
  Systems (IROS)}.\hskip 1em plus 0.5em minus 0.4em\relax IEEE, 2014, pp.
  357--362.

\bibitem{Hua2019Distributed}
Y.~Hua, X.~Dong, G.~Hu, Q.~Li, and Z.~Ren, ``Distributed time-varying output
  formation tracking for heterogeneous linear multi-agent systems with a
  nonautonomous leader of unknown input,'' \emph{IEEE Transactions on Automatic
  Control}, vol.~64, no.~10, pp. 4292--4299, 2019.

\bibitem{Cao2011Formation}
M.~Cao, C.~Yu, and B.~D.~O. Anderson, ``Formation control using range-only
  measurements,'' \emph{Automatica}, vol.~47, no.~4, pp. 776--781, 2011.

\bibitem{Li2018Localization}
R.~Li, Y.~Shi, and Y.~Song, ``Localization and circumnavigation of multiple
  agents along an unknown target based on bearing-only measurement: A three
  dimensional solution,'' \emph{Automatica}, vol.~94, pp. 18--25, 2018.

\bibitem{Shames2012Circumnavigation}
I.~Shames, S.~Dasgupta, B.~Fidan, and B.~D.~O. Anderson, ``Circumnavigation
  using distance measurements under slow drift,'' \emph{IEEE Transactions on
  Automatic Control}, vol.~57, no.~4, pp. 889--903, 2012.

\bibitem{Yu2018Distributed}
X.~Yu, L.~Liu, and G.~Feng, ``Distributed circular formation control of
  nonholonomic vehicles without direct distance measurements,'' \emph{IEEE
  Transactions on Automatic Control}, vol.~63, no.~8, pp. 2730--2737, 2018.

\bibitem{Thien2020Persistently}
T.~M. Nguyen, Z.~Qiu, T.~H. Nguyen, M.~Cao, and L.~Xie, ``Persistently excited
  adaptive relative localization and time-varying formation of robot swarms,''
  \emph{IEEE Transactions on Robotics}, vol.~36, no.~2, pp. 553--560, 2020.

\bibitem{guo2019ultra}
K.~Guo, X.~Li, and L.~Xie, ``Ultra-wideband and odometry-based cooperative
  relative localization with application to multi-uav formation control,''
  \emph{IEEE Transactions on Cybernetics}, vol.~50, no.~6, pp. 2590--2603,
  2019.

\bibitem{Dong2017Time}
X.~Dong and G.~Hu, ``Time-varying formation tracking for linear multiagent
  systems with multiple leaders,'' \emph{IEEE Transactions on Automatic
  Control}, vol.~62, no.~7, pp. 3658--3664, 2017.

\bibitem{Dong2020Flight}
F.~Dong, K.~You, and J.~Zhang, ``Flight control for {UAV} loitering over a
  ground target with unknown maneuver,'' \emph{IEEE Transactions on Control
  Systems Technology}, vol.~26, no.~6, pp. 2461--2473, 2020.

\bibitem{guo2010cooperative}
J.~Guo, G.~Yan, and Z.~Lin, ``Cooperative control synthesis for
  moving-target-enclosing with changing topologies,'' in \emph{2010 IEEE
  International Conference on Robotics and Automation (ICRA)}.\hskip 1em plus
  0.5em minus 0.4em\relax IEEE, 2010, pp. 1468--1473.

\bibitem{dou2021moving}
L.~Dou, X.~Yu, L.~Liu, X.~Wang, and G.~Feng, ``Moving-target enclosing control
  for mobile agents with collision avoidance,'' \emph{IEEE Transactions on
  Control of Network Systems}, vol.~8, no.~4, pp. 1669--1679, 2021.

\bibitem{dong2017necessary}
X.~Dong, Q.~Tan, Q.~Li, and Z.~Ren, ``Necessary and sufficient conditions for
  average formation tracking of second-order multi-agent systems with multiple
  leaders,'' \emph{Journal of the Franklin Institute}, vol. 354, no.~2, pp.
  611--626, 2017.

\bibitem{nguyen2018robust}
T.-M. Nguyen, A.~H. Zaini, C.~Wang, K.~Guo, and L.~Xie, ``Robust
  target-relative localization with ultra-wideband ranging and communication,''
  in \emph{2018 IEEE international conference on robotics and automation
  (ICRA)}.\hskip 1em plus 0.5em minus 0.4em\relax IEEE, 2018, pp. 2312--2319.

\bibitem{Liu2021Anti}
F.~Liu, W.~Meng, and R.~Lu, ``Anti-synchronization of discrete-time fuzzy
  memristive neural networks via impulse sampled-data communication,''
  \emph{IEEE Transactions on Cybernetics}, 2021, DOI:
  10.1109/TCYB.2021.3128903.

\bibitem{Liu2022Bounded}
F.~Liu, W.~Meng, and D.~Yao, ``Bounded antisynchronization of multiple neural
  networks via multilevel hybrid control,'' \emph{IEEE Transactions on Neural
  Networks and Learning Systems}, 2022, DOI: 10.1109/TNNLS.2022.3148194.

\bibitem{Johnstone1982Exponential}
R.~M. Johnstone, C.~R. Johnson, R.~R. Bitmead, and B.~D.~O. Anderson,
  ``Exponential convergence of recursive least squares with exponential
  forgetting factor,'' in \emph{1982 21st IEEE Conference on Decision and
  Control}, 1982.

\bibitem{schmuck2017multi}
P.~Schmuck and M.~Chli, ``Multi-uav collaborative monocular {SLAM},'' in
  \emph{Proceedings of the {IEEE} International Conference on Robotics and
  Automation ({ICRA})}, 2017.

\bibitem{murORB2}
R.~Mur-Artal and J.~D. Tard\'os, ``{ORB-SLAM2}: an open-source {SLAM} system
  for monocular, stereo and {RGB-D} cameras,'' \emph{IEEE Transactions on
  Robotics}, vol.~33, no.~5, pp. 1255--1262, 2017.

\end{thebibliography}

\end{document}